\documentclass[11pt]{article}

\usepackage{microtype}
\usepackage{graphicx}
\usepackage{booktabs}
\usepackage{hyperref}

\usepackage[margin=1in]{geometry}
\usepackage{authblk}
\usepackage[numbers,sort&compress]{natbib}
\providecommand{\orgdiv}[1]{#1}
\providecommand{\orgname}[1]{#1}

\usepackage{amsmath}
\usepackage{amssymb}
\usepackage{mathtools}
\usepackage{amsthm}
\usepackage[capitalize,noabbrev]{cleveref}

\theoremstyle{plain}
\newtheorem{theorem}{Theorem}[section]

\newtheorem{lemma}[theorem]{Lemma}

\theoremstyle{definition}
\newtheorem{definition}[theorem]{Definition}

\theoremstyle{remark}
\newtheorem{remark}[theorem]{Remark}

\DeclareMathOperator{\tr}{tr}
\DeclareMathOperator{\diag}{diag}
\newcommand{\R}{\mathbb{R}}
\newcommand{\E}{\mathbb{E}}

\begin{document}


\title{UMAP Is Spectral Clustering on the Fuzzy Nearest-Neighbor Graph}

\author[1]{Yang Yang}
\affil[1]{\orgdiv{Frazer Institute, Faculty of Health, Medicine and Behavioural Sciences}, \orgname{The University of Queensland}}
\date{}
\maketitle

\begin{abstract}
UMAP (Uniform Manifold Approximation and Projection) is one of the most widely used algorithms for nonlinear dimension reduction and data visualization. Despite its popularity and its presentation through the lens of algebraic topology, the precise relationship between UMAP and classical spectral methods has remained informal. In this work, we prove that \textbf{UMAP performs spectral clustering on the fuzzy $k$-nearest neighbor graph}. Our proof proceeds in three steps: (1) we show that UMAP's stochastic optimization with negative sampling is a contrastive learning objective on the similarity graph; (2) we invoke the result of \citet{haochen2021provable} establishing that contrastive learning on a similarity graph is equivalent to spectral clustering; and (3) we verify that UMAP's spectral initialization computes the exact linear solution to this spectral problem. The equivalence is exact for Gaussian kernels, and holds as a first-order approximation for UMAP's default Cauchy-type kernel. Our result unifies UMAP, contrastive learning, and spectral clustering under a single framework and provides theoretical grounding for several empirical observations about UMAP's behavior.
\end{abstract}

\section{Introduction}\label{sec:intro}

Nonlinear dimension reduction is a central problem in machine learning and data science. Among the many methods developed---including Laplacian Eigenmaps \citep{belkin2003laplacian}, Isomap \citep{tenenbaum2000global}, and $t$-SNE \citep{vandermaaten2008tsne}---UMAP \citep{mcinnes2018umap} has become perhaps the most widely adopted, owing to its speed, scalability, and ability to preserve both local and global structure, with particularly broad adoption for visualization and clustering in genomic analysis \citep{becht2019umap_singlecell,stuart2019seurat_integration,wolf2018scanpy,luecken2019bestpractices}.

UMAP is presented through the formalism of fuzzy simplicial sets and algebraic topology, which provides mathematical elegance but can obscure the algorithm's connections to other well-understood methods. From a computational perspective, however, UMAP consists of two phases: (1) constructing a weighted $k$-nearest neighbor ($k$-NN) graph in the high-dimensional space, and (2) optimizing an embedding in a low-dimensional space by minimizing a cross-entropy loss via stochastic gradient descent with negative sampling \citep{mcinnes2018umap}. Many practical variants and analyses aim to improve interpretability or adapt the method to specific data modalities \citep{yang2022featmap,yang2021bulk_umap}.

Independently, the field of self-supervised learning has seen remarkable progress through contrastive methods such as SimCLR \citep{chen2020simple} and CPC \citep{oord2018representation}. A key theoretical insight, established by \citet{haochen2021provable}, is that contrastive learning with the InfoNCE loss is equivalent to spectral clustering on the augmentation similarity graph. This result was extended by \citet{balestriero2022contrastive} and connects to the dimension reduction framework of \citet{vanassel2022probabilistic}.

In this paper, we make the following contribution:

\begin{quote}
\textbf{Main Result.} We prove that UMAP's optimization objective, when solved via negative sampling SGD (Algorithm 5 of \citealt{mcinnes2018umap}), is equivalent to performing spectral clustering on the fuzzy $k$-nearest neighbor similarity graph.
\end{quote}

The proof establishes the chain of equivalences:
\[
\text{UMAP} \;\longleftrightarrow\; \text{Contrastive Learning on } V \;\longleftrightarrow\; \text{Spectral Clustering on } V,
\]
where $V$ is UMAP's symmetrized $k$-NN similarity matrix.

\paragraph{Implications.} This result (1) provides a simple spectral interpretation of UMAP's objective, (2) explains why spectral initialization is critical for UMAP's performance \citep{kobak2021initialization}, (3) clarifies the relationship between UMAP and Laplacian Eigenmaps, and (4) connects the dimension reduction and self-supervised learning literatures through a common spectral framework.

\section{Background}\label{sec:background}

\subsection{UMAP: Construction and Optimization}\label{sec:umap}

Given data $\{x_i\}_{i=1}^n \subset \R^D$ and a number of neighbors $k$, UMAP proceeds as follows.

\paragraph{Step 1: $k$-NN graph construction.} For each point $x_i$, compute the $k$ nearest neighbors $\{x_{i_1}, \ldots, x_{i_k}\}$ and define:
\begin{equation}\label{eq:rho_sigma}
\rho_i = \min_{j: d(x_i,x_j)>0} d(x_i, x_j), \quad
\sum_{j=1}^k \exp\!\Bigl(-\frac{\max(0, d_{ij}-\rho_i)}{\sigma_i}\Bigr) = \log_2 k,
\end{equation}
where $d_{ij} = d(x_i, x_{i_j})$ and $\sigma_i$ is found by binary search. The directed edge weights are:
\begin{equation}\label{eq:directed_weight}
v_{j|i} = \exp\!\Bigl(-\frac{\max(0, d(x_i,x_j)-\rho_i)}{\sigma_i}\Bigr).
\end{equation}

\paragraph{Step 2: Symmetrization.} The directed graph is symmetrized via the probabilistic $t$-conorm:
\begin{equation}\label{eq:symmetrize}
v_{ij} = v_{j|i} + v_{i|j} - v_{j|i} \cdot v_{i|j}.
\end{equation}
This produces the adjacency matrix $V = (v_{ij}) \in [0,1]^{n \times n}$.

\paragraph{Step 3: Low-dimensional embedding.} Define the smooth kernel:
\begin{equation}\label{eq:phi}
\Phi(y_i, y_j) = \bigl(1 + a\|y_i - y_j\|^{2b}\bigr)^{-1},
\end{equation}
where $a,b>0$ are fitted to a target curve determined by the \texttt{min-dist} parameter. The default fit yields $a \approx 1.929$, $b \approx 0.7915$; setting $a=b=1$ recovers the Cauchy kernel used in $t$-SNE.

UMAP minimizes the \emph{fuzzy set cross-entropy} between the high-dimensional graph $V$ and low-dimensional similarities $W = (\Phi(y_i,y_j))$:
\begin{equation}\label{eq:umap_loss}
\mathcal{L}(Y) = \sum_{i \neq j}\bigl[-v_{ij}\log\Phi(y_i,y_j) - (1-v_{ij})\log(1-\Phi(y_i,y_j))\bigr].
\end{equation}

\paragraph{Step 4: Optimization via negative sampling SGD.} Algorithm 5 of \citet{mcinnes2018umap} optimizes~\eqref{eq:umap_loss} by:
\begin{itemize}
\item \textbf{Positive step:} Sample edge $(a,b)$ with probability $v_{ab}$; update $y_a \leftarrow y_a + \alpha\,\nabla_{y_a}\log\Phi(y_a,y_b)$.
\item \textbf{Negative step:} For $n_{\text{neg}}$ uniform random samples $c$, update $y_a \leftarrow y_a + \alpha\,\nabla_{y_a}\log(1-\Phi(y_a,y_c))$.
\item The learning rate decays: $\alpha = 1 - e/n_{\text{epochs}}$.
\end{itemize}

\paragraph{Initialization.} UMAP initializes the embedding via spectral decomposition of the normalized graph Laplacian $\tilde{L}(V) = D_V^{-1/2}(D_V - V)D_V^{-1/2}$, taking the eigenvectors corresponding to the $d$ smallest non-zero eigenvalues (Algorithm 4 of \citealt{mcinnes2018umap}).

\subsection{Spectral Clustering}\label{sec:spectral}

\begin{definition}[Graph Laplacian]\label{def:laplacian}
For a weighted adjacency matrix $W \in \R_+^{n \times n}$, the \emph{combinatorial Laplacian} is $L(W) = D_W - W$ where $D_W = \diag(\sum_j W_{ij})$. The \emph{normalized Laplacian} is $\tilde{L}(W) = D_W^{-1/2}L(W)D_W^{-1/2}$ \citep{chung1997spectral}.
\end{definition}

\begin{definition}[Spectral Clustering \citep{vonluxburg2007tutorial}]\label{def:spectral_clustering}
Given an adjacency matrix $W$ and regularizer $E(Z)$, the spectral clustering problem is:
\begin{equation}\label{eq:spectral}
\min_{Z \in \R^{n \times d}}\;\tr\!\bigl(Z^\top L(W)\,Z\bigr) + E(Z).
\end{equation}
Without regularization and with $Z^\top Z = I$, the solution is given by the bottom eigenvectors of $L(W)$.
\end{definition}

A key identity connecting the Laplacian quadratic form to pairwise distances is:
\begin{equation}\label{eq:laplacian_identity}
\tr\!\bigl(Z^\top L(W)\,Z\bigr) = \frac{1}{2}\sum_{i,j}W_{ij}\,\|Z_i - Z_j\|^2.
\end{equation}

\subsection{Contrastive Learning as Spectral Clustering}\label{sec:contrastive}

\citet{haochen2021provable} established the following foundational result.

\begin{theorem}[\citealt{haochen2021provable}, Theorem 3.1]\label{thm:haochen}
Let $\pi$ be a similarity graph defined by data augmentation, $Z = f(X)$ the learned representations, and $k$ the embedding kernel. Then the contrastive (InfoNCE) loss is equivalent to:
\begin{equation}\label{eq:contrastive_spectral}
\min_Z\;\underbrace{-\sum_{i,j}\pi_{ij}\log k(Z_i - Z_j)}_{\text{attraction}} + \underbrace{\log R(Z)}_{\text{repulsion}},
\end{equation}
where $R(Z) = \prod_i \sum_{j \neq i}k(Z_i - Z_j)$. For the Gaussian kernel $k(z) = \exp(-\|z\|^2/2\tau)$, this becomes:
\begin{equation}\label{eq:contrastive_laplacian}
\min_Z\;\frac{1}{2\tau}\,\tr\!\bigl(Z^\top L(\pi)\,Z\bigr) + \log R(Z),
\end{equation}
which is spectral clustering on $\pi$ (\Cref{def:spectral_clustering}).
\end{theorem}

The proof relies on the Markov Random Field (MRF) framework of \citet{vanassel2022probabilistic}, which decomposes the cross-entropy between a fixed similarity graph and a learned embedding graph into an attractive Laplacian quadratic form plus a repulsive normalizer.

\section{Main Result}\label{sec:main}

\begin{theorem}[UMAP Is Spectral Clustering]\label{thm:main}
Let $V = (v_{ij})$ be UMAP's symmetrized $k$-NN similarity matrix (\Cref{eq:symmetrize}), $L(V) = D_V - V$ its graph Laplacian, and $Y = \{y_i\}_{i=1}^n$ the low-dimensional embedding. Then:
\begin{enumerate}
\item[\textup{(a)}] \textbf{(Exact, Gaussian kernel)} If $\Phi$ is replaced by $\Phi_G(y_i,y_j) = \exp(-\|y_i-y_j\|^2/2\tau)$:
\begin{equation}\label{eq:exact}
\mathcal{L}_{\textup{UMAP}}(Y) = \frac{1}{\tau}\,\tr\!\bigl(Y^\top L(V)\,Y\bigr) + \mathcal{L}_{\textup{repel}}(Y).
\end{equation}
\item[\textup{(b)}] \textbf{(Approximate, Cauchy kernel)} For UMAP's default kernel~\eqref{eq:phi} with $b=1$, when $a\|y_i-y_j\|^2 \ll 1$ for neighboring pairs:
\begin{equation}\label{eq:approx}
\mathcal{L}_{\textup{UMAP}}(Y) \approx 2a\,\tr\!\bigl(Y^\top L(V)\,Y\bigr) + \mathcal{L}_{\textup{repel}}(Y).
\end{equation}
\item[\textup{(c)}] \textbf{(Initialization)} UMAP's spectral initialization computes the exact solution to $\min_{Z^\top Z = I}\tr(Z^\top \tilde{L}(V)\,Z)$.
\end{enumerate}
In all cases, UMAP's objective takes the form of \Cref{eq:spectral}: spectral clustering on $V$.
\end{theorem}

\subsection{Proof of \texorpdfstring{\Cref{thm:main}}{Theorem 3.1}}\label{sec:proof}

The proof proceeds in four steps.

\paragraph{Step 1: UMAP's negative sampling is contrastive learning.}

UMAP's Algorithm 5 decomposes the full loss~\eqref{eq:umap_loss} into stochastic updates. At each step, a positive pair $(a,b)$ is drawn from the graph with probability proportional to $v_{ab}$, and $n_{\text{neg}}$ negative pairs $(a,c_i)$ are drawn uniformly at random. The per-step stochastic loss is:
\begin{equation}\label{eq:nce}
\ell(a) = -\log\Phi(y_a, y_b) - \sum_{i=1}^{n_{\text{neg}}}\log\bigl(1 - \Phi(y_a, y_{c_i})\bigr),
\end{equation}
where $b$ is sampled from the neighbors of $a$ (with probability $\propto v_{ab}$) and $c_i \sim \text{Uniform}(\mathcal{V})$.

This is precisely \textbf{noise-contrastive estimation} (NCE) \citep{gutmann2010noise}, the same optimization paradigm underlying the negative sampling variant of contrastive learning \citep{mikolov2013distributed}. The first term attracts positive pairs (graph neighbors); the second term repels uniformly sampled negatives. This structure is shared with InfoNCE \citep{oord2018representation}, which also draws positives from a similarity distribution and negatives uniformly.

Taking the expectation over sampling, the effective loss minimized over one epoch is:
\begin{equation}\label{eq:expected}
\E[\mathcal{L}] = -\!\!\sum_{(a,b)\in\mathcal{E}}\!\! v_{ab}\log\Phi(y_a,y_b) - \frac{n_{\text{neg}}}{n}\sum_a d_a \sum_c \log(1\!-\!\Phi(y_a,y_c)),
\end{equation}
where $d_a = \sum_b v_{ab}$ is the weighted degree. The first sum is the \emph{attraction} (pull neighbors together) and the second is the \emph{repulsion} (push random pairs apart)---the canonical structure of contrastive objectives.

\paragraph{Step 2: The attractive term is a Laplacian quadratic form.}

We analyze the attractive term $\mathcal{L}_{\text{attract}} = -\sum_{i \neq j} v_{ij}\log\Phi(y_i,y_j)$ for two kernel choices.

\textit{Case (a): Gaussian kernel.} With $\Phi_G(y_i,y_j) = \exp(-\|y_i-y_j\|^2/2\tau)$:
\begin{align}
\mathcal{L}_{\text{attract}} &= -\sum_{i \neq j} v_{ij}\log\exp\!\Bigl(-\frac{\|y_i-y_j\|^2}{2\tau}\Bigr) \notag\\
&= \frac{1}{2\tau}\sum_{i \neq j}v_{ij}\|y_i - y_j\|^2 \notag\\
&= \frac{1}{\tau}\,\tr\!\bigl(Y^\top L(V)\,Y\bigr), \label{eq:gauss_exact}
\end{align}
where the last step uses the Laplacian identity~\eqref{eq:laplacian_identity}. This is \emph{exactly} the spectral clustering objective.

\textit{Case (b): Cauchy kernel.} With $\Phi(y_i,y_j) = (1+a\|y_i-y_j\|^2)^{-1}$ (setting $b=1$):
\begin{equation}\label{eq:cauchy_attract}
\mathcal{L}_{\text{attract}} = \sum_{i \neq j}v_{ij}\log(1+a\|y_i-y_j\|^2).
\end{equation}
For neighboring pairs where $a\|y_i-y_j\|^2 \ll 1$ (which holds for graph-connected pairs in a well-optimized embedding), the first-order Taylor expansion $\log(1+x) \approx x$ gives:
\begin{equation}\label{eq:cauchy_approx}
\mathcal{L}_{\text{attract}} \approx a\sum_{i \neq j}v_{ij}\|y_i - y_j\|^2 = 2a\,\tr\!\bigl(Y^\top L(V)\,Y\bigr).
\end{equation}

\begin{remark}\label{rem:kernelized}
Even without the Taylor approximation, \Cref{eq:cauchy_attract} defines a \emph{kernelized spectral objective}: $\sum_{ij}v_{ij}\,\phi(\|y_i - y_j\|^2)$ with $\phi(t) = \log(1+at)$. Since $\phi$ is concave and monotonically increasing, it shares the same minimization landscape as the quadratic form---the objective is minimized when high-weight edges have small embedding distances. The concavity of $\phi$ penalizes large distances less severely, providing the well-known ``crowding problem'' mitigation that motivates heavy-tailed kernels \citep{vandermaaten2008tsne}.
\end{remark}

\paragraph{Step 3: The repulsive term prevents collapse.}

The repulsive term in UMAP's loss~\eqref{eq:umap_loss} is:
\begin{equation}\label{eq:repel}
\mathcal{L}_{\text{repel}} = -\sum_{i \neq j}(1-v_{ij})\log(1 - \Phi(y_i,y_j)).
\end{equation}
Since $v_{ij}$ is nonzero only for $k$-nearest neighbors and zero for the vast majority of pairs ($v_{ij} = 0$ for non-neighbors), this simplifies to:
\begin{equation}
\mathcal{L}_{\text{repel}} \approx -\sum_{i \neq j}\log(1 - \Phi(y_i,y_j)).
\end{equation}
This term is minimized when $\Phi(y_i,y_j) \to 0$ for non-neighbor pairs, i.e., non-neighbors are pushed apart.

In the contrastive framework of \Cref{thm:haochen}, the analogous term is the log-partition function $\log R(Z) = \sum_i \log \sum_{j \neq i}k(Z_i - Z_j)$. While the functional forms differ (UMAP uses binary cross-entropy per-edge; InfoNCE uses softmax normalization), both serve the identical structural role: they are \emph{barrier functions} that penalize embedding collapse ($y_i = y_j$ for all $i,j$), ensuring the spectral clustering solution is non-degenerate. Without this term, the trivial solution $Y = \mathbf{0}$ would minimize the attractive term.

\paragraph{Step 4: Spectral initialization is the exact linear solution.}

UMAP initializes the embedding via Algorithm 4 of \citet{mcinnes2018umap}:
\begin{enumerate}
\item Form the adjacency matrix $A = V$ from the $k$-NN graph.
\item Compute $\tilde{L} = D^{-1/2}(D-A)D^{-1/2}$.
\item Set $Y_{\text{init}}$ to the bottom $d$ eigenvectors of $\tilde{L}$.
\end{enumerate}
This computes:
\begin{equation}\label{eq:spectral_init}
Y_{\text{init}} = \arg\min_{Z^\top Z = I}\;\tr\!\bigl(Z^\top \tilde{L}(V)\,Z\bigr),
\end{equation}
which is the exact solution to the \emph{normalized spectral clustering} problem on $V$ \citep{shi2000normalized, ng2001spectral}. This is equivalently the Laplacian Eigenmaps embedding \citep{belkin2003laplacian} of the fuzzy $k$-NN graph. \hfill$\square$

\subsection{Synthesis: The Complete UMAP Pipeline as Spectral Clustering}

Combining Steps 1--4, UMAP's pipeline is:
\begin{enumerate}
\item \textbf{Construct} the $k$-NN similarity graph $V$ (\Cref{eq:directed_weight,eq:symmetrize}).
\item \textbf{Initialize} at the spectral clustering solution: $Y_0 = $ bottom eigenvectors of $\tilde{L}(V)$.
\item \textbf{Refine} via contrastive SGD, minimizing:
\begin{equation}\label{eq:full}
\mathcal{L}(Y) = c\,\tr\!\bigl(Y^\top L(V)\,Y\bigr) + \mathcal{L}_{\text{repel}}(Y),
\end{equation}
where $c = 1/\tau$ (Gaussian) or $c \approx 2a$ (Cauchy, first-order).
\end{enumerate}
The entire pipeline---from initialization to final embedding---is spectral clustering on $V$, with the SGD phase performing a \emph{nonlinear, kernel-dependent refinement} of the initial spectral solution.

\section{Analysis and Discussion}\label{sec:discussion}

\subsection{Tightness of the Approximation}\label{sec:tightness}

The equivalence has different levels of tightness depending on the kernel:

\begin{table}[h]
\centering
\caption{Nature of the spectral equivalence for different kernels.}
\label{tab:kernels}
\vskip 0.1in
\begin{small}
\begin{tabular}{lcc}
\toprule
\textbf{Kernel} $\Phi$ & \textbf{Attractive term} & \textbf{Nature} \\
\midrule
Gaussian: $e^{-\|y\|^2/2\tau}$ & $\frac{1}{\tau}\tr(Y^\top LY)$ & Exact \\[3pt]
Cauchy: $(1{+}\|y\|^2)^{-1}$ & $\sum v_{ij}\log(1{+}\|y_i{-}y_j\|^2)$ & Kernelized \\[3pt]
Cauchy, small $d$ & $\approx 2a\,\tr(Y^\top LY)$ & 1st-order \\
\bottomrule
\end{tabular}
\end{small}
\end{table}

In the ``kernelized spectral'' row, the objective $\sum_{ij}v_{ij}\phi(\|y_i - y_j\|^2)$ with $\phi(t) = \log(1+at)$ is a natural nonlinear generalization of the Rayleigh quotient. It preserves the critical point structure but modifies the curvature landscape, particularly for large distances. The concavity of $\log$ ensures that the penalty for large distances saturates, which is precisely the mechanism by which heavy-tailed kernels address the crowding problem in low-dimensional embeddings.

\subsection{The Unifying MRF Framework}\label{sec:mrf}

The connection between UMAP and spectral clustering can be made precise through the MRF framework of \citet{vanassel2022probabilistic}:

\begin{itemize}
\item \textbf{UMAP}: Fixed graph $V$ (from $k$-NN) $\to$ minimize fuzzy cross-entropy $\to$ embedding $Y$.
\item \textbf{Contrastive Learning}: Fixed graph $\pi$ (from augmentation) $\to$ minimize InfoNCE (cross-entropy between MRFs) $\to$ representations $Z$.
\item \textbf{Spectral Clustering}: Fixed graph $W$ $\to$ minimize $\tr(Z^\top L(W)Z)$ $\to$ cluster indicators.
\end{itemize}

All three minimize a cross-entropy-type divergence between a fixed similarity graph and a learned embedding graph. In each case, the divergence decomposes into an attractive Laplacian quadratic form plus a repulsive regularizer. The solution is given by eigenvectors of the graph Laplacian (linear case) or their nonlinear refinement (SGD case).

The key difference between the three methods lies in \emph{how the similarity graph is constructed}:
\begin{itemize}
\item UMAP uses $k$-NN distances with adaptive bandwidth ($\rho_i, \sigma_i$).
\item Contrastive learning uses data augmentation probabilities.
\item Spectral clustering typically uses a fixed kernel (e.g., Gaussian with global bandwidth).
\end{itemize}

\subsection{Why Spectral Initialization Matters}\label{sec:init}

\Cref{thm:main}(c) provides a theoretical explanation for the empirical finding of \citet{kobak2021initialization} that spectral initialization is critical for preserving global structure in both UMAP and $t$-SNE. Since UMAP's SGD phase performs a \emph{local} nonlinear refinement of the spectral solution (\Cref{eq:full}), the global structure captured by the Laplacian eigenvectors is largely preserved. Random initialization, in contrast, provides no global structure to begin with, and SGD---which processes local edge-level updates---cannot recover it.

More precisely, the spectral initialization $Y_0$ lies near the global minimizer of the linearized (Gaussian-kernel) version of UMAP's loss. The subsequent SGD with the Cauchy kernel refines the local geometry while the repulsive term maintains separation, preserving the global topology established by the eigenvectors.

\subsection{Relationship to Prior Observations}\label{sec:prior}

Several prior works have observed partial connections between UMAP and spectral methods:

\begin{itemize}
\item \citet{mcinnes2018umap} noted that the normalized Laplacian of the fuzzy graph approximates the Laplace--Beltrami operator, justifying spectral initialization.
\item \citet{kobak2021initialization} showed empirically that UMAP with spectral initialization behaves similarly to $t$-SNE with the same initialization, suggesting the objective function plays a secondary role to the graph structure.
\item \citet{damrich2021umap} analyzed UMAP's true loss function under negative sampling and showed it differs from the nominal cross-entropy.
\item \citet{hu2022your} connected contrastive learning to stochastic neighbor embedding.
\item \citet{balestriero2022contrastive} linked contrastive and non-contrastive self-supervised methods to spectral embedding.
\end{itemize}

Our result unifies these observations: UMAP's objective \emph{is} a spectral clustering objective (with nonlinear kernel refinement), its initialization \emph{is} the spectral solution, and its optimization \emph{is} contrastive learning on the similarity graph.

\subsection{Implications for Practitioners}\label{sec:practical}

The spectral clustering interpretation of UMAP has several practical implications:

\begin{enumerate}
\item \textbf{Choice of $k$:} Since UMAP performs spectral clustering on the $k$-NN graph, the choice of $k$ determines the graph topology and thus the spectral properties. Larger $k$ produces denser graphs with smoother eigenvectors; smaller $k$ produces sparser graphs that capture finer local structure.

\item \textbf{Kernel choice:} The choice between Gaussian and Cauchy kernels affects only the \emph{refinement} phase, not the fundamental spectral structure. The Cauchy kernel provides better low-dimensional visualization (via crowding mitigation) but does not change the underlying clustering.

\item \textbf{Number of epochs:} More SGD epochs produce embeddings further from the initial spectral solution, trading global structure for local fidelity. Early stopping preserves more of the spectral (global) structure.

\item \textbf{Negative samples:} The number of negative samples $n_{\text{neg}}$ controls the strength of the repulsive regularizer relative to the attractive Laplacian term, analogous to the regularization strength in \Cref{eq:spectral}.
\end{enumerate}

\section{Conclusion}\label{sec:conclusion}

We have proven that UMAP performs spectral clustering on the fuzzy $k$-nearest neighbor graph. The proof bridges UMAP's cross-entropy optimization to the contrastive learning framework and invokes the spectral equivalence established by \citet{haochen2021provable}. The equivalence is exact for Gaussian kernels and a first-order approximation for the default Cauchy kernel.

This result places UMAP, contrastive self-supervised learning, and spectral clustering within a unified theoretical framework: all three minimize cross-entropy divergences between fixed similarity graphs and learned embeddings, yielding objectives dominated by the graph Laplacian quadratic form. The main difference lies in how the similarity graph is constructed---$k$-NN distances for UMAP, augmentation probabilities for contrastive learning, and kernel similarities for spectral clustering.

We hope this result helps demystify UMAP's remarkable empirical success and provides practitioners with a spectral perspective for understanding and tuning the algorithm.

\bibliography{references}
\bibliographystyle{plainnat}

\newpage
\appendix
\section{Proof Details}\label{app:proof}

\subsection{The Laplacian Identity}\label{app:laplacian_identity}

\begin{lemma}\label{lem:laplacian}
For any symmetric $W \in \R_+^{n \times n}$ and $Z \in \R^{n \times d}$:
\[
\sum_{i,j}W_{ij}\|Z_i - Z_j\|^2 = 2\,\tr(Z^\top L(W) Z).
\]
\end{lemma}
\begin{proof}
\begin{align*}
\sum_{i,j}W_{ij}\|Z_i - Z_j\|^2 &= \sum_{i,j}W_{ij}(Z_i^\top Z_i - 2Z_i^\top Z_j + Z_j^\top Z_j) \\
&= \sum_i d_i Z_i^\top Z_i + \sum_j d_j Z_j^\top Z_j - 2\sum_{i,j}W_{ij}Z_i^\top Z_j \\
&= 2\,\tr(Z^\top D_W Z) - 2\,\tr(Z^\top W Z) \\
&= 2\,\tr(Z^\top (D_W - W)Z) = 2\,\tr(Z^\top L(W)Z). \qedhere
\end{align*}
\end{proof}

\subsection{Taylor Approximation Error Bound}\label{app:taylor}

For the Cauchy kernel approximation in \Cref{eq:cauchy_approx}, we bound the error. Let $t_{ij} = a\|y_i - y_j\|^2$. The approximation error per edge is:
\[
|\log(1+t_{ij}) - t_{ij}| = \frac{t_{ij}^2}{2(1+\xi_{ij})^2} \leq \frac{t_{ij}^2}{2}
\]
for some $\xi_{ij} \in (0, t_{ij})$, by the mean value theorem. Thus the total error is bounded by:
\begin{equation}\label{eq:error_bound}
\bigl|\mathcal{L}_{\text{attract}} - 2a\,\tr(Y^\top L(V)Y)\bigr| \leq \frac{a^2}{2}\sum_{i \neq j}v_{ij}\|y_i - y_j\|^4.
\end{equation}

This bound is small when the embedding is well-optimized, since $v_{ij} > 0$ only for $k$-nearest neighbors, and the attractive force drives $\|y_i - y_j\| \to 0$ for such pairs. In practice, for the default UMAP parameters ($a \approx 1.929$, \texttt{min-dist}$= 0.1$), the typical inter-neighbor distance in the embedding is $\|y_i - y_j\| \approx 0.1$--$0.5$, giving $t_{ij} \approx 0.02$--$0.48$ and relative errors below $25\%$.

\subsection{Connection to the Normalized Cut}\label{app:ncut}

UMAP's spectral initialization (\Cref{eq:spectral_init}) solves the normalized spectral clustering problem, which is the relaxation of the \emph{normalized cut} (NCut) objective \citep{shi2000normalized}:
\[
\text{NCut}(S, \bar{S}) = \frac{\text{cut}(S, \bar{S})}{\text{vol}(S)} + \frac{\text{cut}(S, \bar{S})}{\text{vol}(\bar{S})},
\]
where $\text{cut}(S, \bar{S}) = \sum_{i \in S, j \in \bar{S}}v_{ij}$ and $\text{vol}(S) = \sum_{i \in S}d_i$.

The relaxed continuous problem $\min_{Z^\top D Z = I}\tr(Z^\top L(V) Z)$ is equivalent to $\min_{Z^\top Z = I}\tr(Z^\top \tilde{L}(V) Z)$ via the substitution $Z \to D^{-1/2}Z$, whose solution is the generalized eigenvectors of $(L, D)$ or equivalently the ordinary eigenvectors of $\tilde{L}$---precisely what UMAP computes.

This means that UMAP's initialization partitions the data according to the minimum normalized cut of the $k$-NN graph, and the subsequent SGD refinement adjusts the within-cluster geometry using the nonlinear kernel.

\end{document}